\documentclass{article}

\usepackage{microtype}
\usepackage{graphicx}
\usepackage{subfigure}
\usepackage{booktabs} 
\usepackage{enumitem}

\usepackage{hyperref}

\usepackage[accepted]{icml2023}

\usepackage{amsmath}
\usepackage{amssymb}
\usepackage{mathtools}
\usepackage{amsthm}

\usepackage[capitalize,noabbrev]{cleveref}

\theoremstyle{plain}
\newtheorem{definition}{Definition}[section]
\newtheorem{theorem}{Theorem}
\newtheorem{corollary}{Corollary}[theorem]

\newcommand{\pjn}{{INGNN}}
\newcommand{\mypar}[1]{{\vspace{0.05cm} \noindent \bf #1}}
\newcommand{\bigO}{\mathcal{O}}
\newcommand{\nnz}{\text{nnz}}

\usepackage[textsize=tiny]{todonotes}

\icmltitlerunning{INGNN: Uplifting Message Passing Neural Network with Graph Original Information}

\begin{document}

\twocolumn[
\icmltitle{Uplifting Message Passing Neural Network with Graph Original Information}



\icmlsetsymbol{equal}{*}

\begin{icmlauthorlist}
\icmlauthor{Xiao Liu}{umass}
\icmlauthor{Lijun Zhang}{umass}
\icmlauthor{Hui Guan}{umass}
\end{icmlauthorlist}

\icmlaffiliation{umass}{Department of Computer Science, University of Massachusetts Amherst}

\icmlcorrespondingauthor{Xiao liu}{xiaoliu1990@cs.umass.edu}

\icmlkeywords{Graph Neural Network, Machine Learning}

\vskip 0.3in
]



\printAffiliationsAndNotice{}  

\begin{abstract}
Message passing neural networks (MPNNs) learn the representation of graph-structured data based on graph original information, including node features and graph structures, and have shown astonishing improvement in node classification tasks.
However, the expressive power of MPNNs is upper bounded by the first-order Weisfeiler-Leman test and its accuracy still has room for improvement. 
This work studies how to improve MPNNs' expressiveness and generalizability by fully exploiting graph original information both theoretically and empirically. 
It further proposes a new GNN model called \pjn{} (INformation-enhanced Graph Neural Network) that leverages the insights to improve  node classification performance.  
Extensive experiments on both synthetic and real datasets demonstrate the superiority (average rank 1.78) of our \pjn{} compared with state-of-the-art methods.

\end{abstract}

\section{Introduction}\label{sect:intro}



Graph has been widely used to model structured data such as proteins, social networks, and traffic networks. 
Such graph-structured data usually consists of two types of raw information, \textit{graph structure} that describes connectivity between nodes, and \textit{node features} that describes the attributes of the nodes. We refer to the graph structure and the node features collectively as \textit{graph original information}.
One of the most popular machine learning tasks on graph-structured data is node classification, whose goal is to predict the label, which could be a type or category, associated with all the nodes when we are only given the true labels on a subset of nodes~\cite{hamilton2020graph}.

Graph Neural Networks (GNNs) have been proven to be a powerful approach to learning graph representations for node classification tasks~\cite{hamilton2017inductive,klicpera2018predict,wu2019simplifying}.
The most common type of GNNs adopts a message passing paradigm that recursively propagates and aggregates \textit{node features} through the \textit{graph structure} to produce node representations~\cite{gilmer2017neural,battaglia2018relational}. 
The message passing paradigm, fully relying on the \textit{graph original information} and showing premise of automatic learning, produces much better task performance than traditional graph machine learning that requires deliberate feature engineering.

However, unlike multi-layer feedforward networks (MLPs) which could be universal approximators for any continuous function \cite{hornik1989multilayer}, message passing neural networks (MPNNs) cannot approximate all graph functions \cite{maron2018invariant}. 
MPNNs' expressive power is upper bounded by the first-order Weisfeiler-Leman (1-WL) isomorphism test \cite{xu2018powerful}, indicating that they cannot distinguish two non-isomorphic graph structures if the 1-WL test fails. 
This limitation of MPNNs implies open opportunities to further improve MPNN's performance.  

Researchers have explored several directions in increasing the expressive power of MPNNs. 
As $k$-WL is strictly more expressive than 1-WL, many works tried to mimic the high-order WL tests such as 1-2-3 GNN \cite{morris2019weisfeiler}, PPGN \cite{maron2019provably}, ring-GNN \cite{chen2019equivalence}.
However, they required $O(k)$-order tensors to achieve $k$-WL expressiveness, and thus cannot be generalized to large-scale graphs. 
In order to maintain linear scalability w.r.t. the number of nodes, recent work focused on developing more powerful MPNNs by adding \textit{handcrafted features}, such as subgraph counts \cite{bouritsas2022improving}, distance encoding \cite{li2020distance}, and random features \cite{abboud2020surprising,sato2021random}, to make nodes more distinguishable. 
Although these works achieve good results in many cases, their handcrafted features could introduce inductive bias that hurts generalizability. 
More importantly, they lose the premise of automatic learning~\cite{zhao2021stars}, 
an indispensable property of MPNNs that makes MPNNs more appealing than traditional graph machine learning. 
This trend raises a fundamental question: \textit{is it possible to improve MPNNs' expressiveness, without sacrificing the linear scalability of MPNNs w.r.t. the number of nodes and also their automatic learning property by getting rid of handcrafted features?}


In this work, we confirm the answer to the question and demonstrate that, simply exploiting the \textit{graph original information} rather than deliberately handcrafting features can improve the performance of MPNNs while preserving their linear scalability. The two types of graph original information, graph structures and node features, can be used both together and separately, leading to three features including \textit{ego-node features}, \textit{graph structure features}, and \textit{aggregated neighborhood features}. 
MPNNs, however, leverage only the last feature for downstream tasks. 
We find that all three features are necessary to improve the expressive power or the generalizability of MPNNs and prove it both theoretically and empirically. 
Based on the insight, we further propose an INformation-enhanced MPNN (\pjn{}), which adaptively fuses these three features to derive nodes' representation and achieves outstanding accuracy compared to state-of-the-art baselines.  
We summarize our main contributions as follows:


\setlist{nolistsep}
\begin{itemize}
    \item \textbf{Theoretical Findings.} 
    We show theoretically that the expressiveness of an MPNN that integrates graph structure features is strictly better than $1\&2$-WL and no less powerful than $3$-WL.
    We also prove that the node misclassification rate for graphs in some cases can be depressed by utilizing ego-node features separately from MPNNs.
    We empirically verify the necessity of the three features and show that the node classification accuracy of start-of-the-art models, GCN \cite{kipf2016semi}, H2GCN \cite{zhu2020beyond}, and LINKX \cite{lim2021large} can be improved by adding some of the three identified features that they do not have originally.

    \item \textbf{\pjn{}.} We propose \pjn{}, a powerful GNN which uplifts MPNNs' expressiveness and accuracy by integrating graph structure features and ego-node features. The features are fused automatically with adaptive feature fusion under a bi-level optimization framework.
    Our ablation study demonstrates that all three features, the fusing mechanism, and the optimization procedure play indispensable roles on different types of graphs.
    
    \item \textbf{Evaluation.} 
    We conduct extensive experiments to compare \pjn{} with state-of-the-art GNNs using synthetic and real graph benchmarks that cover the full homophily spectrum. \pjn{} outperforms 12 baseline models and achieves an average rank of 1.78 on 9 real datasets. \pjn{} achieves higher node classification accuracy, $4.17\%$ higher than GCN~\cite{kipf2016semi}, $4.01\%$ higher than MIXHOP \cite{abu2019mixhop}, and $3.43\%$ higher than GPR-GNN~\cite{chien2020adaptive}, on average over the real datasets.
    
\end{itemize}

\section{Related Works}
\label{sec:Related}
\textbf{Regular MPNNs.} Many GNNs \cite{niepert2016learning,hamilton2017inductive,monti2017geometric,velivckovic2017graph,gao2018large,xu2018powerful,wang2019dynamic} fall into the message passing framework \cite{gilmer2017neural,battaglia2018relational}, which iteratively transforms and propagate the messages from the spatial neighborhoods through the graph topology to update the embedding of a target node.
To name a few, GCN \cite{kipf2016semi} designs a layer-wise propagation rule based on a first-order approximation of spectral convolutions on graphs. 
GraphSAGE \cite{hamilton2017inductive} extends GCN by introducing a recursive node-wise sampling scheme to improve the scalability. 
Graph attention networks (GAT) \cite{velivckovic2017graph}  enhances GCN with the attention mechanism \cite{vaswani2017attention}.

\textbf{Improve Expressiveness.}
Several works attempt to improve the expressiveness of GNNs.
The first line of research mimics the high-order WL tests such as 1-2-3 GNN \cite{morris2019weisfeiler}, PPGN \cite{maron2019provably}, ring-GNN \cite{chen2019equivalence}. They usually require exponentially increasing space and time complexity w.r.t. the number of nodes in a graph and thus cannot be generalized to large-scale graphs. 
The second line of research maintains the linear scalability of more powerful GNNs by adding features, such as subgraph counts \cite{bouritsas2022improving}, distance encoding \cite{li2020distance}, and random features \cite{abboud2020surprising,sato2021random}. 
The handcrafted features may introduce inductive bias that hurts generalization and also lose the premise of automatic learning \cite{zhao2021stars}.
The third line of work tries to design more expressive GNNs by involving high-order neighbors.
For example, H2GCN \cite{zhu2020beyond} identifies ego- and neighbor-embedding separation, second-order neighbors, and the combination of intermediate representations to boost representation learning.
MixHop \cite{abu2019mixhop} proposes a graph convolutional layer that utilizes multiple powers of the adjacency matrix to learn general mixed neighborhood information.
Similarly, Generalized PageRank (GPR-) GNN \cite{chien2020adaptive} adaptively controls the contribution of different propagation steps.
In this paper, we propose to improve MPNNs' expressiveness by integrating the original graph structure features so that we can retain the linear complexity and the automatic learning property of MPNNs.

\textbf{Improve Generalizability.}
Recent works start to pay attention to improving the GNNs' generalizability, including overcoming over-smoothing and handling graphs with various homophily.
DAGNN \cite{liu2020towards}  proposes to decouple the transformation and the propagation operation to increase receptive fields.
APPNP \cite{klicpera2018predict} leverages personalized PageRank to improve the propagation scheme.
Later, BernNet \cite{he2021bernnet} learns an arbitrary spectral filter such as the ones in GCN, DAGNN, and APPNP via Bernstein polynomial approximation.
Similarly, AdaGNN \cite{dong2021adagnn} introduces trainable filters to capture the varying importance of different frequency components.
Considering graph heterophily, ACM-GCN \cite{luan2022revisiting} proposes a multi-channel mixing mechanism, enabling adaptive filtering at nodes with different homophily.
And GloGNN \cite{li2022finding} handles low-homophilic graphs with global homophily.
To improve MPNN's generalizability, we propose to enhance MPNNs with ego-node features, and show that it can depress the node misclassification rate in some graph cases.


\section{Improve MPNNs and Theory} \label{sec:theory}
This section first introduces the notations and then describes how we can exploit the graph original information to improve the expressiveness and the generalizability of MPNNs. 

\textbf{Notations.}
Let $\mathcal{G}=(\mathcal{V},\mathcal{E},\mathbf{X})$ denotes a graph with $N$ nodes and $M$ edges, where $\mathcal{V}$ and $\mathcal{E}$ are the set of nodes and edges respectively, $|\mathcal{V}|=N$ and $|\mathcal{E}|=M$. 
We use $\mathbf{A} \in \{0,1\}^{N \times N}$ as the adjacency matrix where $\mathbf{A}[i,j]=1$ if $(v_i,v_j) \in \mathcal{E}$ otherwise $\mathbf{A}[i,j]=0$. 
Each node $v_i \in \mathcal{V}$ has a raw feature vector $x_i$ of size $D$. 
The raw feature vectors of all nodes in the graph form a feature matrix $\mathbf{X} \in \mathbb{R}^{N \times D}$. We refer to the adjacency matrix $\mathbf{A}$ and the feature matrix $\mathbf{X}$ together as \textit{graph original information}.

We attempt to improve MPNNs by fully exploiting graph original information. 
Specifically, given a graph $\mathcal{G}$, we investigate the three features $f_{ego}(\mathbf{X})$, $f_{strc}(\mathbf{A})$, and $f_{aggr}(\mathbf{X},\mathbf{A})$, which result from using the node features $\mathbf{X}$ and the adjacency matrix $\mathbf{A}$ together and separately. 
$f_{ego}(\cdot)$ and $f_{strc}(\cdot)$ can be MLPs for simulating any continuous function, and $f_{aggr}(\cdot)$ is an MPNN.
For the description purpose, we refer to the output of $f_{ego}(\mathbf{X})$ as \textit{ego-node features} since it only takes node feature itself as input, outputs of $f_{strc}(\mathbf{A})$ as \textit{graph structure features} since it only looks at the graph structure, and outputs of $f_{aggr}(\mathbf{X},\mathbf{A})$ as \textit{aggregated neighborhood features}. 

\subsection{Improve Expressiveness.}
We find that integrating the graph structure features to MPNNs, i.e., combining $f_{strc}(\mathbf{A})$ and $f_{aggr}(\mathbf{X},\mathbf{A})$, can improve the expressiveness of MPNNs.

\begin{figure}[t]
\centering
\subfigure[(4,4)-rook graph] 
{
\includegraphics[width=0.33\linewidth]{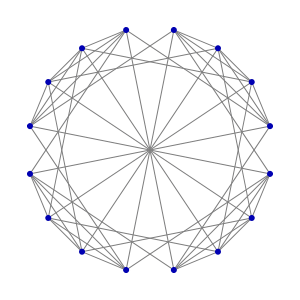}\label{fig:44rook}
\raisebox{0.25\height}{\includegraphics[width=0.12\linewidth]{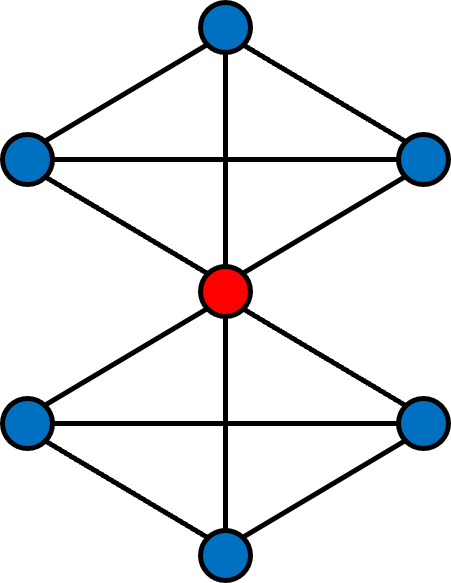}\label{fig:sub44rook}}
}
\subfigure[Shrikhande graph]
{
\includegraphics[width=0.33\linewidth]{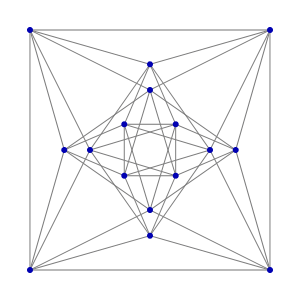}\label{fig:Shrikhande}
\raisebox{0.25\height}{\includegraphics[width=0.12\linewidth]{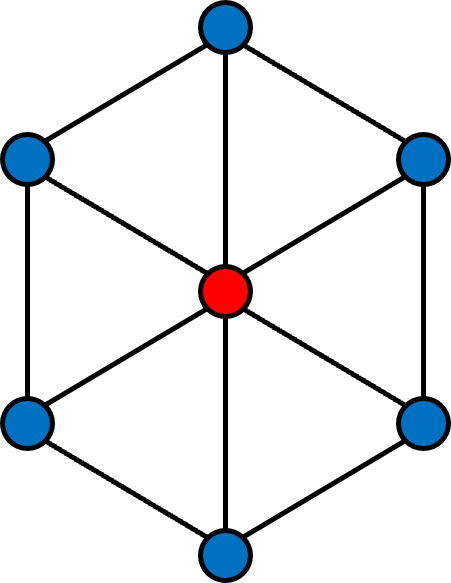}\label{fig:subShrikhande}
}}
\caption{Two non-isomorphic strongly regular graphs that cannot be distinguished by 3-WL. The right side is the neighborhood subgraph of an arbitrary node .}
\label{fig:srg}
\end{figure}



\begin{theorem} \label{thm:powerful}
The expressive power of the combination of the aggregated neighborhood features and the graph structure feature, i.e., $\mathcal{F}_{comb}(\mathbf{X},\mathbf{A})=[f_{aggr}(\mathbf{X},\mathbf{A}), f_{strc}(\mathbf{A})]$, where $f_{strc}$ is injective functions, and $f_{aggr}$ is an MPNN with a sufficient number of layers that are also injective functions, is strictly more powerful than 1\&2-WL, and no less powerful than 3-WL. 
\end{theorem}
\vspace{-10pt}
\begin{proof}[Proof of Theorem \ref{thm:powerful}] 
We first prove that $\mathcal{F}_{comb}$ is at least as powerful as 1-WL. 
In other words, if two graphs are identified by $\mathcal{F}_{comb}$ as isomorphic, then it is also identified as isomorphic by 1-WL. 
Given two graphs $\mathcal{G}_1$ and $\mathcal{G}_2$,  if they are identified as isomorphic by $\mathcal{F}_{comb}$, it means that $\mathcal{F}_{comb}$ maps them to the same representation. 
Then we know that every part of the two representations is the same as well, e.g., $f_{aggr}(\mathbf{X}_{\mathcal{G}_1},\mathbf{A}_{\mathcal{G}_1})=f_{aggr}(\mathbf{X}_{\mathcal{G}_2},\mathbf{A}_{\mathcal{G}_2})$, meaning that $f_{aggr}$ cannot distinguish the two graphs. Since $f_{aggr}$ is as powerful as 1-WL \cite{maron2019provably, azizian2020expressive}, $\mathcal{G}_1$ and $\mathcal{G}_2$ cannot be distinguished by 1-WL either.

We then prove that $\mathcal{F}_{comb}$ is no less powerful than 3-WL, which means that there exist some graphs that could be distinguished by $\mathcal{F}_{comb}$ but not by 3-WL \cite{chen2020can}. 
We prove it using the two graphs in Figure~\ref{fig:srg}, which are strongly regular graphs with the same graph parameters $(16,6,2,2)$ \footnote{Strongly regular graphs are regular graphs with $v$ nodes and degree $k$. And every two adjacent vertices have $\lambda$ common neighbors and every two non-adjacent vertices have $\mu$ common neighbors. The tuple ($v, k, \lambda, \mu$ ) is the parameters of a strongly regular graph.}. These two graphs cannot be distinguished by 3-WL \footnote{Any strongly regular graphs with the same parameters cannot be distinguished by 3-WL \cite{arvind2020weisfeiler}.}. 
The figure also shows an example neighborhood subgraph around an arbitrary node marked as red. 

Now we prove that $\mathcal{F}_{comb}$ can successfully distinguish the graphs in Figure \ref{fig:44rook} and \ref{fig:Shrikhande}, due to the use of graph structure features.
Notice that all 1-hop subgraphs for each node in Figure \ref{fig:44rook} (also in Figure \ref{fig:Shrikhande}) are the same. It implies that $\mathcal{F}_{comb}$ can distinguish them as long as it can distinguish one subgraph pair as shown on the right side of Figure \ref{fig:44rook} and \ref{fig:Shrikhande}.
We can easily prove that $\mathcal{F}_{comb}$ can distinguish the two subgraphs, namely they are non-isomorphic, since their adjacency matrices are non-identical under any neighborhood permutation.
A simple algorithm for proving it is to enumerate all the neighborhood permutations, then check if there is one pair of adjacency matrices that are the same under some permutation. In this case, the enumeration algorithm will return False. 
Given that the adjacency matrices of the two subgraphs are different with any possible node permutation, $\mathcal{F}_{comb}$ will identify the two subgraphs as non-isomorphic because $f_{strc}$ is an injective function and will map the adjacency matrices of the two subgraphs into different representations. 
Therefore, Figure \ref{fig:44rook} and \ref{fig:Shrikhande} can be distinguished as well, since all subgraph pairs are the same and can be distinguished.
Hence, we have proven that $\mathcal{F}_{comb}$ is no less powerful than 3-WL.


As 3-WL is strictly more powerful than 1-WL \cite{sato2020survey} and 1-WL and 2-WL have equivalent expressiveness \cite{maron2019provably}, we have shown that $\mathcal{F}_{comb}$ is strictly more powerful than 1\&2-WL and no less powerful than 3-WL.
\end{proof}

\subsection{Improve Generalizability.}
We also find that integrating ego-node features into MPNNs, i.e., combining $f_{ego}(\mathbf{X})$ and $f_{aggr}(\mathbf{X},\mathbf{A})$, can improve the generalization ability of MPNNs. Specifically, we will prove that the node features after aggregation might be harder to classify under some graph homophily settings, and that using ego-node features can depress the node misclassification rate in these cases. Since the proof is related to graph homophily, we first formally define Graph Homophily $\mathcal{H}$ that measures the overall similarity between the nodes connected by an edge in terms of the labels.

\begin{definition}[Graph Homophily] \label{def:Homophily}
Given a graph $\mathcal{G}$ with node labels $\mathcal{Y}$, the edge homophily is defined as
$\mathcal{H}(\mathcal{G},\mathcal{Y})=\frac{1}{|\mathcal{E}|}\sum_{(u,v) \in \mathcal{E}} \mathrm{}{1}(y_u=y_v)$, which represents the fraction of the edges that connect two nodes with the same class label.
\end{definition}

\begin{theorem} \label{thm:misclassification}
There exists a graph homophily range that the misclassification rate increases after aggregation.
\end{theorem}
\vspace{-10pt}
\begin{proof}[Proof of Theorem \ref{thm:misclassification}] 
The key of this proof is to identify a homophily range and show that, if a graph falls into this range, classifying with node features after aggregation will increase the misclassification rate compared with using raw node features. 
To prove that there must exist such a homophily range, we identify three specific homophily cases where the misclassification rate changes (increases or decreases) when using aggregated node features. Then based on the intermediate value theorem, we can prove the existence of the homophily range with an increased misclassification rate.
Without loss of generality, we prove the theorem using a binary node classification problem.

We first introduce necessary definitions related to the problem.
Suppose that there is a regular graph $\mathcal{G}$ with homophily $\mathcal{H}$ and degree $d$. The nodes $\mathcal{V}$ in $\mathcal{G}$ can be categorized into two classes, and the feature of a node $v$ from each class is sampled from two normal distribution $N(\mu_{1},\sigma_{1}^{2})$ and $N(\mu_{2},\sigma_{2}^{2})$ respectively. We call the probability density function of the two distributions $f_1(x)$ and $f_2(x)$.
An optimal classifier for a binary node classification problem will categorize a node to a class that the node feature has a higher probability to be sampled from its distribution. 
The misclassification rate of such an optimal classifier $\epsilon$ is the overlapping area of $f_1(x)$ and $f_2(x)$,
\begin{equation}
\begin{aligned}
\epsilon = 1 - \phi_1(z) + \phi_2(z), \\
\end{aligned}
\label{eq:misrate}
\end{equation}
where $\phi_1$ and $\phi_2$ are the two cumulative distribution functions and $z$ is the solution of $f_1(x)=f_2(x)$. In particular, we denote the misclassification rate of using the raw node features as $\epsilon_{raw}$.

Next, we discuss how the two node feature distributions change after aggregation, which will affect the misclassification rate. 
We consider the case that the neighborhood node features are aggregated by an averaging function.
Since the graph homophily is $\mathcal{H}$, for an arbitrary node, there are $\mathcal{H}d$ neighbors from the same class and $(1-\mathcal{H})d$ from the other class in average.  
The aggregated features of the nodes from class 1 follow the distribution $N(\mathcal{H}\mu_1 + (1-\mathcal{H})\mu_2, \frac{\mathcal{H}\sigma_1^2 + (1-\mathcal{H})\sigma_2^2}{d})$. Similarly, that from class 2 follow $N(\mathcal{H}\mu_2 + (1-\mathcal{H})\mu_1, \frac{\mathcal{H}\sigma_2^2 + (1-\mathcal{H})\sigma_1^2}{d})$.

With the above problem settings, we can now identify 3 special cases $\mathcal{H} \in \{1.0,0.0,0.5\}$ whose misclassification rate based on aggregated node features becomes higher or lower compared to $\epsilon_{raw}$.
\begin{itemize}
    \item \textbf{Case 1:} If $\mathcal{H} = 1.0$, the distribution of aggregated node features are $N(\mu_1, \frac{\sigma_1^2}{d})$ and $N(\mu_2, \frac{\sigma_2^2}{d})$ for class 1 and 2 respectively. Compared to the original feature distributions, the mean values are the same but the variances are reduced, leading to a smaller overlapping area of the distributions. In other words, the misclassification rate is reduced, i.e., $\epsilon_{\mathcal{H} = 1.0}<\epsilon_{raw}$.

    \item \textbf{Case 2:} If $\mathcal{H} = 0.0$, the distribution of aggregated node features are $N(\mu_2, \frac{\sigma_2^2}{d})$ and $N(\mu_1, \frac{\sigma_1^2}{d})$ for class 1 and 2 respectively. Compared to Case 1, the distributions are exchanged, which means the misclassification rate should be the same as $\mathcal{H}=1.0$, namely $\epsilon_{\mathcal{H} = 0.0}=\epsilon_{\mathcal{H} = 1.0}<\epsilon_{raw}$.

    \item \textbf{Case 3:} If $\mathcal{H} = 0.5$, the distribution of aggregated node features are $N(\frac{\mu_1+\mu_2}{2}, \frac{\sigma_1^2+\sigma_2^2}{2d})$ for both class 1 and 2. In this case, the two distributions are indistinguishable for any classifier, i.e., $\epsilon_{\mathcal{H} = 0.5}=1.0$. 
\end{itemize}

Figure \ref{fig:misrate} shows the misclassification rate when using the aggregated node features of the graphs with different $\mathcal{H}$. Since the misclassification rate (Eq. \ref{eq:misrate}) is a continuous function and $\epsilon_{\mathcal{H} = 0.0}=\epsilon_{\mathcal{H} = 1.0}<\epsilon_{raw}<\epsilon_{\mathcal{H} = 0.5}$, according to the intermediate value theorem, there exists $0.0<\mathcal{H}_l<0.5$ and $0.5<\mathcal{H}_u<1.0$, such that for all $\mathcal{G}$ with homophily $\mathcal{H} \in (\mathcal{H}_l,\mathcal{H}_u)$, i.e. green area in Figure \ref{fig:misrate}, the misclassification rate of the aggregated node features is higher than the one of the raw features. 
In other words, there exists a graph homophily range that the misclassification rate increases after aggregation. Moreover, the above proof could be generalized to multiple classes, higher feature dimensions, and different aggregation functions.
\end{proof}

\begin{figure}[t]
\centering
\includegraphics[width=0.7\linewidth]{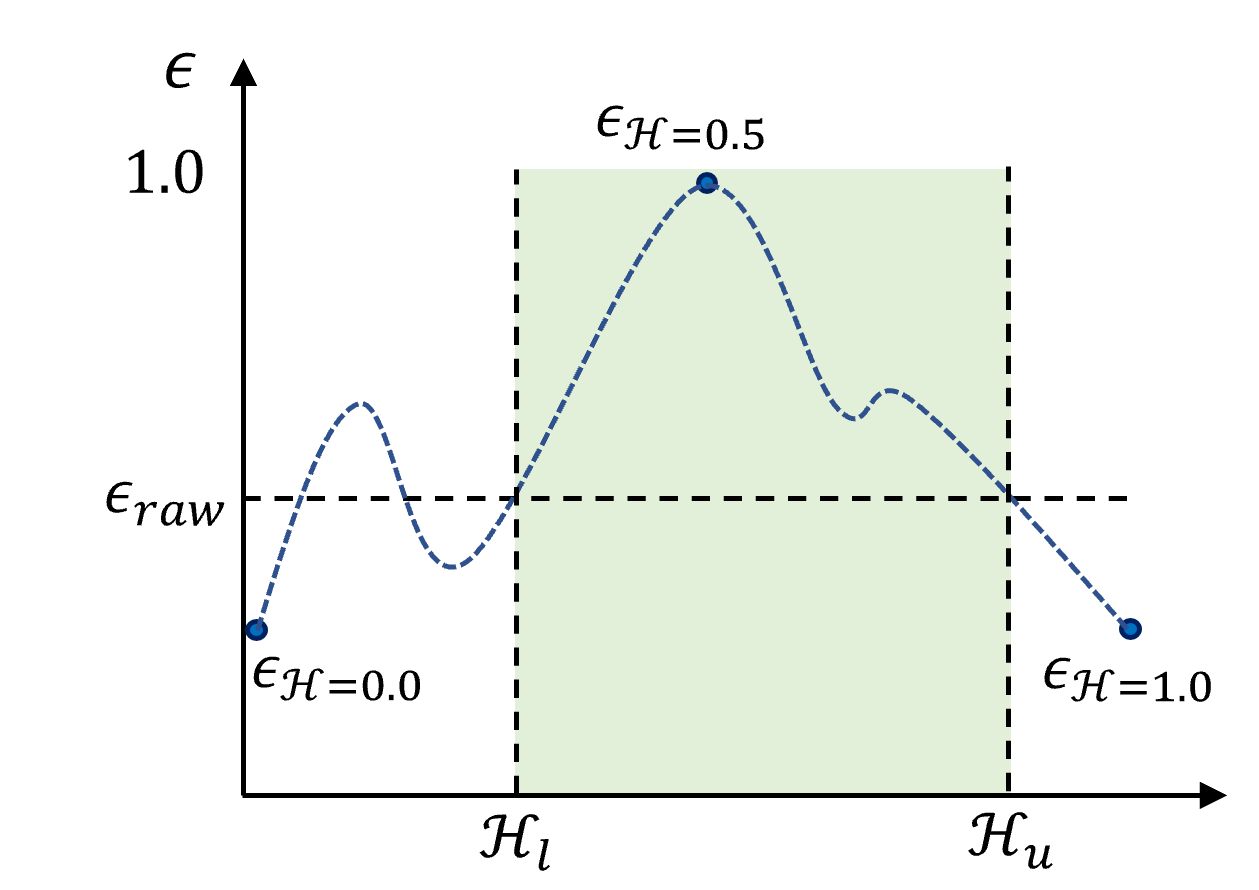}
\vspace{-10pt}
\caption{The misclassification rate of different homophily.}
\label{fig:misrate}
\end{figure}
\vspace{-8pt}

\begin{corollary} \label{cor:ego}
Combining ego-node features with the aggregated neighborhood feature, i.e., $[f_{aggr}(\mathbf{X},\mathbf{A}), f_{ego}(\mathbf{X})]$ can depress the misclassification rate in the graph homophily range of Theorem \ref{thm:misclassification}.
\end{corollary}
\vspace{-10pt}
\begin{proof}[Proof of Corollary \ref{cor:ego}] 
Since we know that the misclassification rate of using ego-node features $f_{ego}(\mathbf{X})$ is $\epsilon_{raw}$, and that of using aggregated neighborhood features $f_{aggr}(\mathbf{X},\mathbf{A})$ is larger than $\epsilon_{raw}$ when the graph homophily is in the range proved in Theorem \ref{thm:misclassification}, combining $f_{ego}(\mathbf{X})$ could reconcile the negative effects of $f_{aggr}(\mathbf{X},\mathbf{A})$. In other words, it can depress the misclassification rate in this case.
\end{proof}


To summarize, Theorem \ref{thm:powerful} and Corollary \ref{cor:ego} reveal that combining $f_{ego}(\mathbf{X})$ and $f_{strc}(\mathbf{A})$ with $f_{aggr}(\mathbf{X},\mathbf{A})$ will improve the expressiveness and generalizability of MPNNs.

\section{Information-Enhanced GNN (INGNN)} \label{sec:ModelDesign}
Based on the analysis in Section \ref{sec:theory}, we design \pjn{} which achieves high expressive power and good generalizability solely based on the graph original information. \pjn{} transforms the graph original information into ego-node features, graph structure features, and aggregated neighborhood features, and then adaptively fuse them to derive nodes' representation.  We explain its two main components, \textit{feature extractors} and \textit{adaptive feature fusion} and then the \textit{training strategy} in detail. 



\textbf{Feature Extractors.} 
To extract ego-node features, we use a simple linear transformation of the node feature matrix $\mathbf{X} \in \mathbb{R}^{N \times D}$:
$\mathbf{H}_\text{ego} = \mathbf{X} \mathbf{W}_\text{ego},$
where $\mathbf{W}_\text{ego} \in \mathbb{R}^{D \times d}$ is the transformation matrix and $\mathbf{H}_\text{ego} \in \mathbb{R}^{N \times d}$ is the extracted ego-node features. We use linear transformation instead of MLP because linear transformation consistently achieves the best accuracy in our empirical evaluation. 

To extract the aggregated neighborhood features, we adopt the state-of-the-art design in the message passing framework, where feature transformation and propagation are decoupled~\cite{zeng2021decoupling}. 
It propagates the ego-node features $\mathbf{H}_\text{ego}$ and combines the features of different propagation steps to enlarge the receptive field: 
$\mathbf{H}_\text{agg} = \sum_{i=1}^{s_1} \hat{\mathbf{A}}^{i} \mathbf{H}_\text{ego},$
where $\hat{\mathbf{A}} = \mathbf{D}^{- \frac{1}{2}} \mathbf{A} \mathbf{D}^{- \frac{1}{2}}$ is the normalized adjacency matrix, and $\mathbf{D}$ is the diagonalized node degrees. $s_1$ is the maximum propagation step. 


To extract the graph structure features, we compute the combination of different powers of the adjacency matrix after a linear transformation: 
$\mathbf{H}_\text{strc} = \sum_{j=1}^{s_2} \mathbf{A}^{j} \mathbf{W}_\text{strc},
$
where $\mathbf{W}_\text{strc} \in \mathbb{R}^{N \times d}$ is the transformation matrix, and $s_2$ is the maximum power of the adjacency matrix. 
We use a simple linear transformation for its efficiency and we do not observe performance gains from an MLP.  

\textbf{Adaptive Feature Fusion.}
The adaptive feature fusion module assigns a trainable scalar importance score for each feature so that it can \textit{automatically} learn the features' importance from the input graph: 
\begin{equation}
\begin{aligned}
\mathbf{H} &= \sigma(\pi_1 \mathbf{H}_\text{ego} + \pi_2 \mathbf{H}_\text{agg} + \pi_3 \mathbf{H}_\text{strc}), \\
\pi_i &= \frac{\exp{p_i}}{\sum_{j=1}^{3} \exp{p_j}},
\end{aligned}
\end{equation}

where $\mathbf{H}$ is the fused feature and $\sigma$ is a activation function ReLU. $\mathcal{P}=\{ p_i | i=1,2,3 \}$ are trainable parameters, and $\{ \pi_i | i=1,2,3\}$ are the weights for each feature.


After obtaining the fused feature $\mathbf{H}$, we predict the labels for each node with a linear classifier,
$\mathbf{Y}_\text{pred} = \text{Softmax}(\mathbf{H} \mathbf{W}_\text{pred})$, where $\mathbf{Y}_\text{pred} \in \mathbb{R}^{N \times C}$ is the predictions and $\mathbf{W}_\text{pred} \in \mathbb{R}^{d \times C}$ is the predictor's parameters.


\textbf{Training with Bi-level Optimization.}
To train our model parameters and feature fusion weights jointly, we borrow the idea of bi-level optimization \cite{liu2018darts,dong2019searching}. 
Suppose the model parameters is $\mathbf{W}$ and the parameters for feature fusion is $\mathcal{P}$. 
The loss function of the node classification tasks is:
\begin{equation}
  \mathcal{L}(\mathbf{W},\mathcal{P},\mathcal{G},\mathbf{X},\mathbf{Y}) = - \frac{1}{|\mathbf{Y}|} \sum_{y_i \in \mathbf{Y}} y_i^\text{T} log(\hat{y}_i),
\end{equation}
where $\mathcal{G}$ is the graph, $\mathbf{X}$ and $\mathbf{Y}$ are raw node features and ground-truth labels respectively, and $\hat{y}_i \in \mathbf{Y}_\text{pred}$ is the prediction of our model for node $i$. Then the objective of our bi-level optimization is:
\begin{equation}
\begin{aligned}
  &\min_{\mathcal{P}} \mathcal{L}_\text{valid}(\mathbf{W}^{*},\mathcal{P},\mathcal{G},\mathbf{X}_\text{valid},\mathbf{Y}_\text{valid}),\\
  \text{s.t.} \quad 
  &\mathbf{W}^{*} = \arg\min_{\mathbf{W}} \mathcal{L}_\text{train}(\mathbf{W},\mathcal{P},\mathcal{G},\mathbf{X}_\text{train},\mathbf{Y}_\text{train}).
\end{aligned}
\end{equation}

In short, we optimize the model parameters $\mathbf{W}$ on the train set, while optimizing the feature fusion parameters $\mathcal{P}$ on the validation set alternatively.

\mypar{Time Complexity.} The time complexity of extracting the ego-node features is $\bigO(\nnz(\mathbf{X}) \cdot d)$, where $\nnz(\mathbf{X})$ is the number of non-zero values in the node feature $\mathbf{X}$, $d$ is the hidden feature dimension. Extracting the aggregated neighborhood features takes $\bigO(s_1 M N d)$ to aggregate features from $s_1$-hop neighbors, where $M$ and $N$ are the number of edges and nodes in the graph. Similarly, extracting the graph structure features takes $\bigO(s_2 M N d)$ with sparse matrix multiplications. The feature fusion step takes $\bigO(N d)$ for reweighting and summing up the features. Finally, the linear classifier takes $\bigO(N d C)$ to do the predictions.

\section{Experiments} \label{sec:Experiments}

%

\subsection{Experimental Settings} \label{sec:ExpSet}
\textbf{Datasets.}
We conduct experiments on both synthetic datasets and real datasets to examine the efficacy of our model in terms of test accuracy. 
We generate synthetic graphs \texttt{syn-cora} with the approach in H2GCN~\cite{zhu2020beyond}. 
The \texttt{syn-cora} dataset provides 11 graphs with homophily ranging from 0.0 to 1.0 with 0.1 as the interval.
The raw node features and labels for each graph are sampled from the \texttt{cora} dataset~\cite{sen2008collective}. 
The edges of the graph are generated gradually according to the given homophily.  
We evaluate the average test accuracy over five trials on these graphs with the official data splits.

We also evaluate our method and existing GNNs on 9 real-world datasets. 
Table~\ref{tab:dataset-info} in Appendix Section~\ref{sect:app-dataset-details} summarizes the statistics of the datasets. 
\texttt{Cora} \cite{sen2008collective}, \texttt{CiteSeer} \cite{sen2008collective}, \texttt{PubMed} \cite{sen2008collective}, and \texttt{Coauthor CS \& Physics} \cite{shchur2018pitfalls} are widely-used homophilic graphs, while \texttt{penn94} \cite{traud2012social}, \texttt{arXiv-year} \cite{hu2020open}, \texttt{genius} \cite{lim2021expertise}, and \texttt{twitch-gamer} \cite{rozemberczki2021twitch} are larger scale graphs with relatively low homophily \cite{lim2021new,lim2021large}. 
We follow the official data splits as explained in Appendix Section~\ref{sect:app-dataset-details}. We report the average and the standard deviation of the test accuracy in the following experiments.

\textbf{Baselines for Comparison.}
Our baselines include \textit{MLP-based methods} (MLP, LINKX \cite{lim2021large}), \textit{regular MPNNs} (GCN \cite{kipf2016semi} and GAT \cite{velivckovic2017graph}), \textit{MPNNs with improved generalizability} (DAGNN \cite{liu2020towards}, AdaGNN \cite{dong2021adagnn}, BerNet \cite{he2021bernnet}, ACM-GCN \cite{luan2022revisiting}, and GloGNN \cite{li2022finding}), and \textit{high-order MPNNs} with higher expressive power (H2GCN \cite{zhu2020beyond}, MIXHOP \cite{abu2019mixhop}, and GPR-GNN \cite{chien2020adaptive}). 
We did grid-based hyperparameter search for all baselines and our approach. Hyperparamters are reported in  Appendix Section~\ref{app:hyperparameter}. 

\textbf{Hardware Specifications.} We run experiments on both synthetic and real-world benchmarks with a 12-core CPU, 8 GB Memory, and an NVIDIA GeForce GTX 1080 Ti GPU with 11 GB GPU Memory for all the methods except for H2GCN, because it suffers from the out-of-memory (OOM) problem. For H2GCN, we use a workstation with a 12-core CPU, 32 GB Memory, and an NVIDIA Quadro RTX 8000 GPU with 48 GB GPU Memory.


\subsection{Performance of INGNN} \label{sec:Results}
\textbf{Accuracy on Different Datasets.}
Table \ref{tab:syn-cora-results} reports the average test accuracy over five random splits on the graphs in the \texttt{syn-cora} dataset. 
Overall, \pjn{} outperforms the existing methods on most datasets, with six at top-1 and four at top-2. Especially, \pjn{} outperforms it on a larger range of homophily ($0.2 \sim 0.8$), which are also more common cases in real-world datasets. It pushes the best accuracy up to $2.76\%$. A full version with standard deviation can be found in Appendix Section \ref{sect:app-synthetic-details}. 

\begin{table*}[h]
\centering
\caption{Test accuracy of different methods on the graphs with different homophily in \texttt{syn-cora} dataset. {\color{red} \textbf{Red}} and {\color{blue} blue} represent top-1 and top-2 ranking in terms of accuracy respectively.}
\label{tab:syn-cora-results}
\scriptsize
\tabcolsep=0.12cm
\begin{tabular}{c|ccccccccccc}
\toprule
synh   & 0    & 0.1    & 0.2    & 0.3    & 0.4    & 0.5    & 0.6    & 0.7    & 0.8    & 0.9    & 1    \\
\midrule
\midrule
MLP	& $69.20$	  & $69.20$	 & $69.20$	 & $69.20$	 & $69.20$	  & $69.20$	 & $69.20$	 & $69.20$	 & $69.20$	 & $69.20$	 & $69.20$	 \\
LINKX  & $72.09$	 & $70.54$	 & $69.76$	 & $70.13$	 & $71.34$	 & $74.53$	 & $77.16$	 & $80.35$	 & $83.30$	 & $87.59$	 & $89.60$	 \\
GCN	& $28.61$	 & $30.64$	 & $36.03$	 & $45.15$	 & $51.39$	  & $65.04$	 & $74.48$	 & $82.22$	 & $91.21$	 & $96.19$	 & $99.92$	 \\
GAT	& $29.41$	 & $30.32$	 & $34.83$	 & $43.86$	 & $51.15$	 & $64.80$	 & $74.34$	 & $81.45$	 & $90.46$	 & $95.79$	 & {\color[HTML]{FF0000} $\mathbf{100.0}$} \\
DAGNN  & $34.32$	 & $39.49$	 & $45.01$	 & $54.48$	 & $60.51$	 & $72.36$	 & $80.00$	 & $86.49$	 & $93.32$	 & {\color[HTML]{FF0000} $\mathbf{97.48}$} & {\color[HTML]{0000FF} $99.95$}	  \\
H2GCN  & {\color[HTML]{FF0000} $\mathbf{76.43}$} & {\color[HTML]{FF0000} $\mathbf{73.86}$} & {\color[HTML]{0000FF} $71.58$}	  & {\color[HTML]{0000FF} $72.17$}	  & {\color[HTML]{0000FF} $72.95$}	  & {\color[HTML]{0000FF} $78.31$}	  & {\color[HTML]{0000FF} $83.27$}	  & {\color[HTML]{FF0000} $\mathbf{87.43}$} & $92.09$	 & $97.00$	 & $98.98$	 \\
MIXHOP & $39.44$	 & $38.95$	 & $41.05$	 & $48.93$	 & $55.09$	 & $64.75$	 & $74.45$	 & $82.44$	 & $91.45$	 & $96.25$	 & {\color[HTML]{FF0000} $\mathbf{100.0}$} \\
GPR-GNN & $67.86$	 & $61.96$	 & $61.21$	 & $64.69$	 & $67.67$	 & $74.56$	 & $80.19$	 & $86.68$	 & {\color[HTML]{0000FF} $93.54$}	  & {\color[HTML]{0000FF} $97.45$}	  & {\color[HTML]{FF0000} $\mathbf{100.0}$} \\ \hline 
\pjn{}   & {\color[HTML]{0000FF} $72.49$}	  & {\color[HTML]{0000FF} $73.67$}	  & {\color[HTML]{FF0000} $\mathbf{73.11}$} & {\color[HTML]{FF0000} $\mathbf{74.32}$} & {\color[HTML]{FF0000} $\mathbf{75.71}$} & {\color[HTML]{FF0000} $\mathbf{79.49}$} & {\color[HTML]{FF0000} $\mathbf{84.67}$} & {\color[HTML]{0000FF} $87.32$}	  & {\color[HTML]{FF0000} $\mathbf{93.70}$} & $94.75$	 & {\color[HTML]{0000FF} $99.95$}	 \\
\bottomrule

\end{tabular}
\vspace{-10pt}
\end{table*}

Table \ref{tab:real-results} reports the results on the 9 real datasets. Overall, \pjn{} achieves five top-1 and three top-2 over 9 datasets, with an average rank of 1.78.
\pjn{} could always achieve accuracy improvement from $3.05\% \sim 30.94\%$ compared with MLP, $1.17\% \sim 10.51\%$ compared with GCN, $0.99\% \sim 7.16\%$ compared with GAT. While compared with MLP-based LINKX, \pjn{} obtains better accuracy in 8 out of 9 datasets.
Then for MPNNs with better generalizability, i.e, DAGNN, AdaGNN, BerNet, ACM-GCN, and GloGnn, we achieve competitive accuracy on heterophilic graphs (i.e., \texttt{penn94}, \texttt{arXiv-year}, \texttt{genius}, and \texttt{twitch-gamer}) and outperforms them on the other datasets.
As for MPNNs with higher expressiveness, i.e., H2GCN, MIXHOP, and GPR-GNN, the accuracy of \pjn{} are far better in all the datasets with different settings.

\begin{table*}[htb]
\centering
\caption{Average test accuracy $\pm$ standard deviation on the real datasets. 
{\color{red} \textbf{Red}} and {\color{blue} blue} represent top-1 and top-2 ranking in terms of accuracy respectively. OOM means a model runs out of memory on a specific dataset.}
\label{tab:real-results}
\scriptsize
\tabcolsep=0.01cm
\begin{tabular}{c|ccccccccc|c}
\toprule
      & arXiv-year       & penn94       & twitch-gamer     & genius       & CiteSeer     & PubMed       & Cora     & Coauthor CS      & Coauthor Physics     & Avg. Rank \\
Homophily & 0.22     & 0.47     & 0.55     & 0.62     & 0.74     & 0.80     & 0.81     & 0.81     & 0.93     & -      \\
\#Nodes   & 169,343      & 41,554       & 168,114      & 421,961      & 3,327    & 19,717       & 2,708    & 18,333       & 34,493       & -      \\
\#Edges   & 1,166,243    & 1,362,229    & 6,797,557    & 984,979      & 4,552    & 44,324       & 5,278    & 81,894       & 247,962      & -      \\
\#Classes & 5    & 2    & 2    & 2    & 6    & 3    & 7    & 15       & 5    & -      \\
\midrule
\midrule
MLP       & $36.70_{\pm 0.21}$   & $73.6_{\pm 0.40}$     & $60.92_{\pm 0.07}$   & $86.68_{\pm 0.09}$     & $50.94_{\pm 4.20}$   & $66.04_{\pm2.29}$    & $52.56_{\pm2.55}$  & $83.08_{\pm1.00}$    & $82.15_{\pm±5.11}$       &  12.22      \\
LINKX     & {\color[HTML]{0000FF} $56.00_{\pm 1.34}$} & $84.71_{\pm 0.52}$ & $66.06_{\pm 0.19}$ & {\color[HTML]{0000FF} $90.77_{\pm 0.27}$} & $53.66_{\pm 3.69}$    & $67.66_{\pm 4.29}$    & $62.66_{\pm 2.12}$    & $88.53_{\pm 1.43}$    & $89.37_{\pm 1.52}$    &   7.56     \\
\midrule
GCN       & $46.02_{\pm 0.26}$    & $82.47_{\pm 0.27}$    & $62.18_{\pm 0.26}$    & $87.42_{\pm 0.37}$    & $63.36_{\pm 2.06}$    & $78.12_{\pm 1.60}$    & $77.90_{\pm 1.18}$    & $90.35_{\pm 0.88}$    & $92.39_{\pm 0.89}$    &   7.78     \\
GAT       & $49.37_{\pm 0.20}$    & $81.45_{\pm 0.55}$    & $62.32_{\pm 0.23}$    & $86.59_{\pm 1.06}$    & $65.90_{\pm 1.88}$    & $76.78_{\pm 2.38}$    & $76.98_{\pm 1.75}$    & $88.86_{\pm 0.65}$    & $92.57_{\pm 0.60}$    &    7.67    \\ \midrule
DAGNN     & $38.49_{\pm 0.28}$    & $74.84_{\pm 0.52}$    & $60.36_{\pm 0.14}$    & $71.11_{\pm 9.11}$    & {\color[HTML]{0000FF} $67.12_{\pm 1.71}$} & $78.28_{\pm 1.58}$ & {\color[HTML]{0000FF} $82.34_{\pm 1.42}$} & {\color[HTML]{0000FF} $91.83_{\pm 0.72}$} & $93.22_{\pm 0.77}$    &  6.78     \\
AdaGNN     & $49.49_{\pm 0.16}$    & $83.55_{\pm 0.31}$    & $64.64_{\pm 0.27}$    & $89.68_{\pm 0.81}$    & $63.44_{\pm 1.94}$ & $76.80_{\pm 1.45}$ & $80.66_{\pm 1.07}$ & $90.80_{\pm 0.84}$ & $92.66_{\pm 0.89}$    &   5.33   \\
BernNet    & $36.49_{\pm 0.18}$    & $82.81_{\pm 0.51}$    & $62.37_{\pm 0.21}$    & $88.83_{\pm 0.68}$    & $53.52_{\pm 4.70}$ & $75.32_{\pm 1.55}$ & $76.52_{\pm 2.97}$ & $91.43_{\pm 0.91}$ & $91.30_{\pm 0.86}$    &   8.44    \\
ACM-GCN    & $48.41_{\pm 0.30}$    & $82.68_{\pm 0.60}$    & $61.48_{\pm 0.61}$    & $81.19_{\pm 6.15}$    & $65.08_{\pm 1.91}$ & $76.70_{\pm 1.32}$ & $78.90_{\pm 1.66}$ & $90.50_{\pm 0.54}$ & $92.70_{\pm 0.74}$    &   7.33    \\
GloGNN     & $54.52_{\pm 0.39}$    & {\color[HTML]{FF0000} $\mathbf{85.60_{\pm 0.27}}$}    & {\color[HTML]{FF0000} $\mathbf{66.34_{\pm 0.29}}$}   &  {\color[HTML]{FF0000} $\mathbf{90.91_{\pm 0.13}}$}  & $55.72_{\pm 3.09}$ & $72.72_{\pm 1.16}$ & $74.70_{\pm 1.62}$ & $90.50_{\pm 1.29}$ & $89.16_{\pm 2.61}$    &  6.33     \\
\midrule
H2GCN     & $49.09_{\pm 0.10}$    & $81.54_{\pm 0.56}$      & OOM      & OOM      & $64.40_{\pm 1.44}$    & $76.30_{\pm 2.80}$    & $79.24_{\pm 1.75}$    & $91.18_{\pm 0.58}$    & {\color[HTML]{0000FF} $93.56_{\pm 0.48}$}    &  5.86    \\
MIXHOP    & $51.78_{\pm 0.26}$    & $83.63_{\pm 0.54}$    & $65.65_{\pm 0.30}$    & $90.61_{\pm 0.24}$ & $56.98_{\pm 4.80}$    & $76.14_{\pm 2.37}$    & $73.80_{\pm 4.02}$    & $89.79_{\pm 0.91}$    & $93.33_{\pm 0.75}$    &  6.44      \\
GPR-GNN   & $44.89_{\pm 0.20}$    & $81.12_{\pm 0.63}$    & $62.00_{\pm 0.25}$    & $90.02_{\pm 0.13}$    & $64.72_{\pm 1.59}$    & {\color[HTML]{0000FF} $79.12_{\pm 0.87}$}    & $80.44_{\pm 1.53}$    & $90.74_{\pm 0.60}$    & {\color[HTML]{FF0000} $\mathbf{93.86_{\pm 0.36}}$} &   5.78     \\
\midrule
\pjn  &{\color[HTML]{FF0000} $\mathbf{56.53_{\pm 0.15}}$} & {\color[HTML]{0000FF} $85.44_{\pm 0.60}$} & {\color[HTML]{0000FF} $66.07_{\pm 0.11}$} & $89.73_{\pm 0.51}$    & {\color[HTML]{FF0000} $\mathbf{69.82_{\pm 1.08}}$} & {\color[HTML]{FF0000} $\mathbf{79.98_{\pm 1.57}}$} & {\color[HTML]{FF0000} $\mathbf{83.50_{\pm 0.93}}$} & {\color[HTML]{FF0000} $\mathbf{93.15_{\pm 0.36}}$} & {\color[HTML]{0000FF} $93.56_{\pm 0.65}$} & 1.78 \\
\bottomrule
\end{tabular}
\vspace{-10pt}
\end{table*}
  
\mypar{Ablation Study.} 
We present an ablation study to show the effectiveness of our design choices: the three graph features, adaptive feature fusion, and bi-level optimization. 
Table \ref{tab:ablation} reports the accuracy results from five variants of our model by removing one design element at a time.

\textit{Graph features.} The rows \textit{w/o egg}, \textit{w/o agg}, and \textit{w/o strc} demonstrate the contribution of the three features extracted from the original graph information on the model's accuracy.
Overall, models without one of the features suffer from $2.11\% \sim 5.43\%$ accuracy drop on all the datasets on average, indicating the importance of these features on graph representation learning regardless.
Specifically, models without the aggregated neighborhood features have a larger accuracy drop on homophilic graphs (i.e., \texttt{Cora}, \texttt{CiteSeer}, \texttt{PubMed}, \texttt{Coauthor CS \& Physics}). It echoes the high performance of MPNNs (e.g., GCN and DAGNN) on homophilic graphs. 
On the contrary, models without the graph structure features suffer from severe accuracy drops, up to $20.52\%$, especially on heterophilic graphs. It echoes the expressive power benefits of integrating the graph structure features as proved in Theorem \ref{thm:powerful}.   
Moreover, models without the ego-node features have accuracy drops from $0.38\% \sim 4.31\%$, which echoes its depression effect on the misclassification rate as stated in Theorem \ref{thm:misclassification}.

\textit{Adaptive feature fusion.} Models without adaptive feature fusion treat the three graph features equally and sum them up without re-weighting to produce the final node feature. It suffers from an accuracy drop of $1.84\%$ on average, indicating the importance of adaptive feature fusion. Another widely used approach for feature fusion is to concatenate the features. Although concatenation could learn separated parameters for different features, the features could not be balanced well, leading to similar unsatisfying results as summing up the features (see Appendix Section \ref{sec:app_ablation_concat}). 

\textit{Bi-level optimization.} 
Without bi-level optimization, the feature fusion weights are jointly optimized with the model parameters on the training dataset. Its accuracy drops by $2.27\%$ on average. 
This phenomenon is consistent with the observation in the NAS domain \cite{liu2018darts}: training model parameters and feature fusion weights jointly on the same training set would cause over-fitting and thus poor generalization performance.

\begin{table*}[htb]
\centering
\caption{Ablation studies on the real datasets.}
\label{tab:ablation}
\scriptsize
\tabcolsep=0.08cm
\begin{tabular}{l|lllllllll|c}
\toprule
             &\begin{tabular}[c]{@{}l@{}}arXiv\end{tabular}  & penn94 & 
             \begin{tabular}[c]{@{}l@{}}twitch\end{tabular} & genius & CiteSeer & PubMed & Cora  & \begin{tabular}[c]{@{}l@{}} CS\end{tabular} & \begin{tabular}[c]{@{}l@{}}  Physics\end{tabular} & $\Delta$ Avg. \\
Homophily    & 0.22       & 0.47   & 0.55        & 0.62  & 0.74    & 0.80  & 0.81  & 0.81                                                   & 0.93                                                      & -          \\
\midrule
\midrule
\textbf{\pjn{}} & \textbf{56.53} & \textbf{85.44} & \textbf{66.07} & \textbf{89.72} & \textbf{69.82} & \textbf{79.98} & \textbf{83.50} & \textbf{93.15}   & \textbf{93.56}                                               & -          \\
w/o ego      & 53.80  \textsubscript{2.74\textdownarrow}    & 81.83 \textsubscript{3.62\textdownarrow}  & 65.68   \textsubscript{0.40\textdownarrow}     & 85.61 \textsubscript{4.12\textdownarrow} & 68.44  \textsubscript{1.38\textdownarrow}  & 78.92 \textsubscript{1.06\textdownarrow} & 83.12\textsubscript{0.38\textdownarrow} & 88.84  \textsubscript{4.31\textdownarrow}                                                 & 92.56   \textsubscript{1.00\textdownarrow}                                                     & 2.11\textdownarrow       \\
w/o agg      & 53.17  \textsubscript{3.36\textdownarrow}    & 84.81 \textsubscript{0.63\textdownarrow} & 65.78  \textsubscript{0.30\textdownarrow}      & 89.78 \textsubscript{0.05\textuparrow} & 57.54 \textsubscript{12.28\textdownarrow}   & 72.32 \textsubscript{7.66\textdownarrow}  & 72.66 \textsubscript{10.84\textdownarrow} & 87.72 \textsubscript{5.43\textdownarrow}                                                  & 85.14 \textsubscript{8.42\textdownarrow}                                                      & 5.43\textdownarrow       \\
w/o strc     & 36.01 \textsubscript{20.52\textdownarrow}     & 75.50 \textsubscript{9.94\textdownarrow} & 61.59  \textsubscript{4.49\textdownarrow}      & 86.87 \textsubscript{2.86\textdownarrow} & 69.82 \textsubscript{0.00\textdownarrow}  & 81.12 \textsubscript{0.36\textuparrow} & 83.34 \textsubscript{0.16\textdownarrow} & 92.46    \textsubscript{0.69\textdownarrow}                                               & 93.41   \textsubscript{0.15\textdownarrow}                                                     & 4.35\textdownarrow       \\
w/o fusion   & 53.45  \textsubscript{3.08\textdownarrow}    & 84.35 \textsubscript{1.09\textdownarrow} & 65.76  \textsubscript{0.32\textdownarrow}     & 87.64 \textsubscript{2.09\textdownarrow} & 67.72 \textsubscript{2.10\textdownarrow}  & 76.04 \textsubscript{3.94\textdownarrow} & 82.48 \textsubscript{1.02\textdownarrow} & 91.13  \textsubscript{2.02\textdownarrow}                                                 & 92.64   \textsubscript{0.92\textdownarrow}                                                     & 1.84\textdownarrow       \\
w/o bi-level & 53.27  \textsubscript{3.26\textdownarrow}    & 83.71 \textsubscript{1.74\textdownarrow} & 65.74  \textsubscript{0.34\textdownarrow}      & 88.25 \textsubscript{1.48\textdownarrow} & 68.60 \textsubscript{1.22\textdownarrow}   & 73.32 \textsubscript{6.66\textdownarrow} & 81.02 \textsubscript{2.48\textdownarrow} & 90.31 \textsubscript{2.84\textdownarrow}                                                  & 93.15 \textsubscript{0.41\textdownarrow}                                                       & 2.27\textdownarrow      \\
\bottomrule
\end{tabular}
\vspace{-10pt}
\end{table*}

\begin{table*}[t]
\centering
\caption{Average test accuracy $\pm$ standard deviation when adding features on GCN, LINKX, and H2GCN.}
\label{tab:addfeature}
\scriptsize
\tabcolsep=0.01cm
\begin{tabular}{c|ccccccccc|c}
\toprule
      & arXiv-year       & penn94       & twitch-gamer     & genius       & CiteSeer     & PubMed       & Cora     & CS     & Physics     & Avg. Improv.  \\
Homophily & 0.22     & 0.47     & 0.55     & 0.62     & 0.74     & 0.80     & 0.81     & 0.81     & 0.93     & -      \\
\midrule
\midrule
GCN $+\mathbf{H}_\text{ego}+\mathbf{H}_\text{strc}$      & $54.60_{\pm 0.34}$    & $82.85_{\pm 0.51}$    & $65.43_{\pm 0.11}$    & $89.01_{\pm 0.97}$    & $63.80_{\pm 1.93}$    & $77.32_{\pm 1.17}$    & $78.52_{\pm 1.99}$    & $90.41_{\pm 1.25}$    & $91.75_{\pm 1.60}$    &  1.50   \\
H2GCN $+\mathbf{H}_\text{strc}$    & $50.71_{\pm 0.15}$    & $83.96_{\pm 0.53}$    & OOM      & OOM      & $65.18_{\pm 2.13}$    & $75.24_{\pm 2.36}$    & $78.04_{\pm 1.10}$    & $92.03_{\pm 0.43}$    &  $92.97_{\pm 0.45}$    &     0.40 \\
LINKX $+\mathbf{H}_\text{aggr}$    & $55.13_{\pm 0.26}$ & $84.79_{\pm 0.21}$ & $66.08_{\pm 0.24}$ & $91.17_{\pm 0.16}$ & $64.88_{\pm 2.03}$    & $75.72_{\pm 1.72}$    & $81.26_{\pm 1.71}$    & $91.99_{\pm 0.46}$    & $93.11_{\pm0.56}$    &  4.97\\  
\bottomrule
\end{tabular}
\vspace{-10pt}
\end{table*}

\begin{figure}[!htb]
\centering
\subfigure[Results on \texttt{syn-cora} dataset.] 
{
\includegraphics[width=0.75\linewidth]{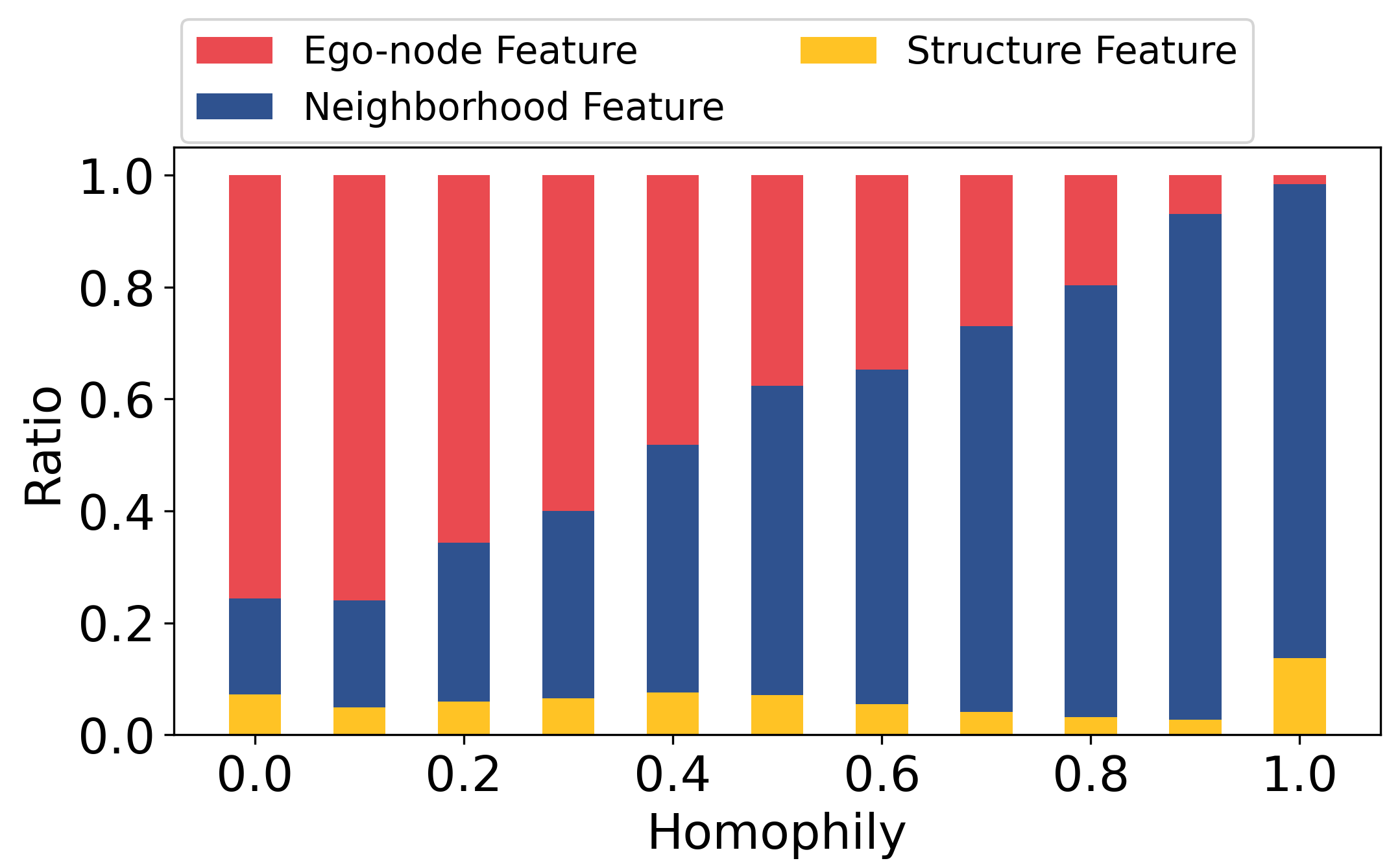}\label{fig:feature_proportion_syn}
}
\subfigure[Results on real datasets.]{
\includegraphics[width=0.75\linewidth]{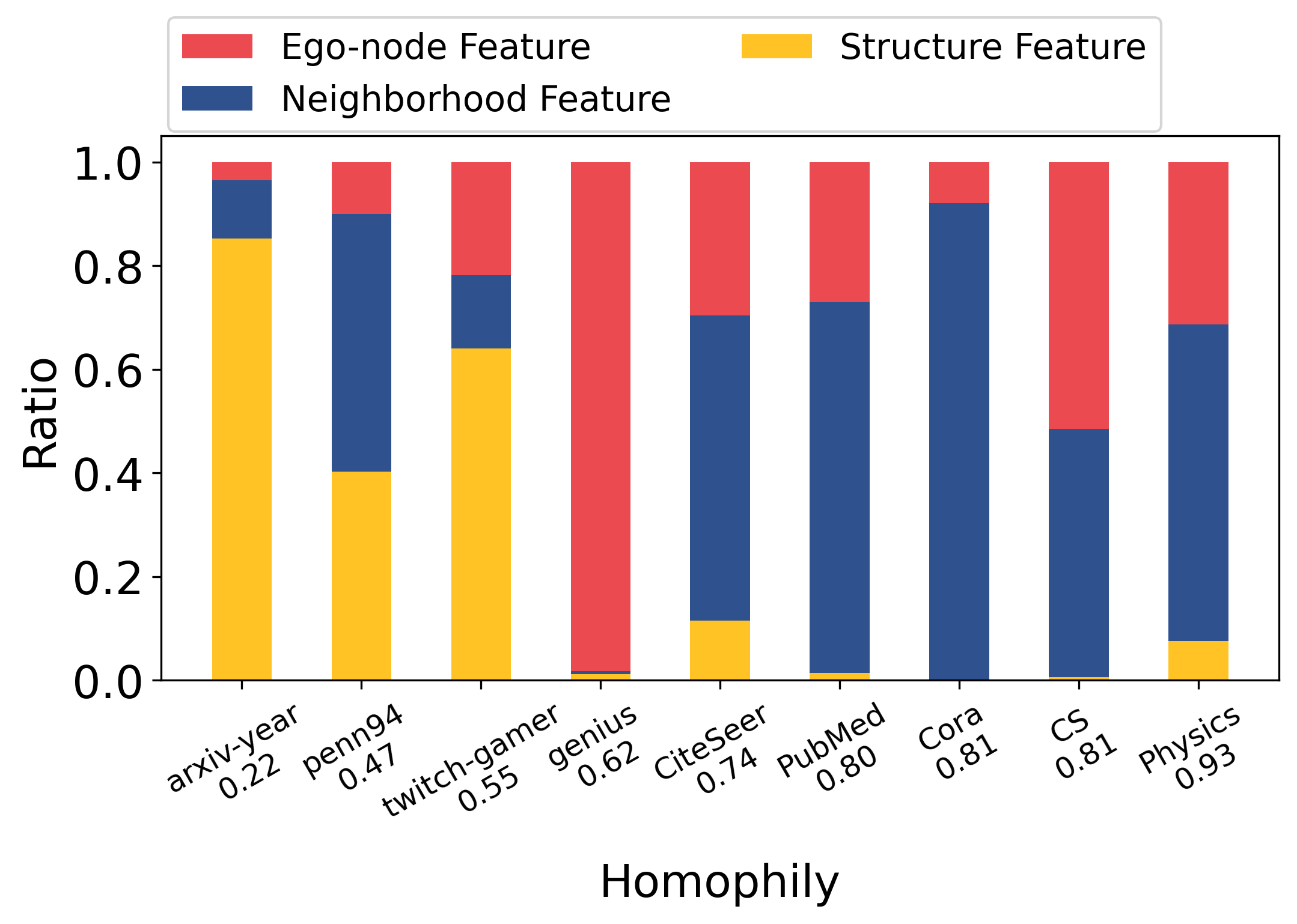}\label{fig:feature_proportion_real}
}
\vspace{-10pt}
\caption{Importance of the features after feature fusion.}
\label{fig:feature_proportion}
\end{figure}

\textbf{Insights from Feature Fusion.} \label{sec:FusionAnalysis}
We measure the importance of the three graph features by computing their proportion after feature fusion. Formally, the importance score is computed as
 $\mathcal{I}_{s} = \frac{\pi_{i} \langle \mathbf{H}_{s} \rangle}{\pi_1 \langle \mathbf{H}_\text{ego} \rangle + \pi_2 \langle \mathbf{H}_\text{agg} \rangle + \pi_3 \langle \mathbf{H}_\text{strc} \rangle},
$
where $s \in \{ego,agg,strc\}$ and $\langle \cdot \rangle$ computes the averaged absolute value of a specific feature.

Figures~\ref{fig:feature_proportion_syn} show the results on \texttt{syn-cora}. We observe three dominant trends.
(1) The importance of the aggregated neighborhood features increases from 0.2 to 0.8 as the homophily increases, echoing the intuition that a node has similar properties to its neighbors on homophilic graphs.
(2) The importance of the ego-node features increases from close to around 0.05 to 0.8 as the homophily of a graph becomes lower. This indicates that the ego-node features are a more reliable signal than other features in some cases, which is consistent with Theorem \ref{thm:misclassification}. 
(3) All the features have a non-trivial importance score for graphs with homophily within 0.2 and 0.8 (which is common for real graphs). This indicates the importance of all the features in learning node embeddings, echoing insights from the ablation study.
Since the graph links in \texttt{syn-cora} are generated randomly, which means that the graph structure features cannot capture meaningful information from the adjacency matrix as we expected, it's hard to figure out the trend of this feature from Figure \ref{fig:feature_proportion_syn}.

Figure \ref{fig:feature_proportion_real} shows the results on real graphs. 
Since real graphs have different intrinsic graph properties, including the raw node features and the graph links, we cannot compare the changes in feature importance across graphs like what we did for \texttt{syn-cora}. Instead, we focus on comparing the importance of different features given specific graphs. Our observations are summarized as follows.
(1) For homophilic graphs (i.e., \texttt{CiteSeer}, \texttt{PubMed}, \texttt{Cora}, \texttt{Coauthor-CS}, and \texttt{Coauthor-Physics}), the aggregated neighborhood features play the most important role to the node classification accuracy. This phenomenon echoes what we observe in \texttt{syn-cora}.
(2) For graphs with low to medium homophily (i.e., \texttt{arXiv-year}, \texttt{penn94}, and \texttt{twitch-gamer}), the graph structure features take the biggest proportion compared to the other two features, indicating the strong impact of this feature. It is consistent with our ablation study where removing the graph structure features causes the most severe accuracy drops.
(3) The graph \texttt{genius} almost totally relies on ego-node features. We suspect that the reason is that its node features are sufficient to serve the node classification task. 
As shown in Table \ref{tab:real-results}, applying MLP on the raw node features already forms a strong baseline for this graph.

\subsection{Benefits of Adding Features to Existing Methods.} \label{sec:AddAnalysis}
We add the three identified features into existing methods, GCN, LINKX, and H2GCN, to show the benefits of these features. Specifically, since GCN only has the aggregated neighborhood feature, the ego node features and the graph structure features are added by two separated branches like in \pjn. For LINKX, we introduce the aggregated neighborhood features since it is an MLP-based method with the other two features already. As for H2GCN, the graph structure features is introduced as an additional branch. Besides, the new features of GCN and H2GCN are adaptively fused with the original features with bi-level optimization. We follow LINKX's design to concatenate different features.

Table \ref{tab:addfeature} shows the average accuracy for the three improved methods as well as the average accuracy improvement compared with its original performance. Although the newly added features are combined with the existing methods just in a trivial way, they demonstrate undeniable positive effects on the test accuracy, i.e., $0.40\% \sim 4.97\%$ improvement on average. The accuracy improvement achieves up to $8.57\%$ on GCN, $18.60\%$ on LINKX, and $2.42\%$ on H2GCN.

\section{Conclusion}\label{sect:conclusion}
This paper studied how to fully exploit the graph original information to improve MPNNs' performance. 
We showed theoretically that integrating the graph structure features with MPNNs can improve their expressive power and integrating ego-node features can depress the node misclassification rate in some cases.
We further proposed a graph neural network \pjn{} that leverages the insights to improve node classification performance.
Extensive experiments show that \pjn{} achieves state-of-the-art accuracy compared with strong baselines on both synthetic and real datasets. 

\small
\bibliographystyle{icml2023}
\bibliography{sample}



\newpage
\appendix
\onecolumn
\section{Real-world Datasets Details}\label{sect:app-dataset-details}

In our experiments, we use the following real-world datasets to evaluate our method and existing GNNs. The homophily of these datasets ranges from $0.222 \sim 0.931$. Table \ref{tab:dataset-info} summarizes the detailed dataset statistics. 

\texttt{Cora}~\cite{sen2008collective}, \texttt{CiteSeer}~\cite{sen2008collective}, \texttt{PubMed}~\cite{sen2008collective}, and \texttt{Coauthor CS \& Physics}~\cite{shchur2018pitfalls} have high edge homophily and are usually considered as homophilic graphs. For these graphs, we follow the data split in GCN~\cite{kipf2016semi} and DAGNN~\cite{liu2020towards}. For \texttt{Cora}, \texttt{CiteSeer}, and \texttt{PubMed}, we randomly sample 20 nodes from each class as the train set, and sample 500 nodes from the rest as the validation set and 1000 nodes as the test set. For \texttt{Coauthor CS \& Physics}, we randomly sample 20 nodes per class as the train set, 30 nodes per class as the validation set, and the rest nodes as the test set.

\texttt{penn94} \cite{traud2012social}, \texttt{arXiv-year} \cite{hu2020open}, \texttt{genius} \cite{lim2021expertise}, and \texttt{twitch-gamer} \cite{rozemberczki2021twitch} are graphs with lower homophily. For these graphs, we follow the data split in LINKX \cite{lim2021new,lim2021large}, which uses the $50\%$/$25\%$/$25\%$ nodes as the train/validation/test set respectively.

For all the datasets, we generate 5 random data splits for computing the average and standard deviation of the models' performance.

\begin{table}[htb]
\centering
\caption{Statistics for the real-world datasets.}
\label{tab:dataset-info}
\vskip 0.1in
\small
\tabcolsep=0.02cm
\begin{tabular}{c|ccccccccccc}
\toprule
                                                            & \#Classes & \#Nodes & \#Edges   & \#Features & degree & \begin{tabular}[c]{@{}c@{}}edge\\ homophily\end{tabular} & Nodes        & Edges       & Classes \\
\midrule
\midrule
Cora                                                        & 7         & 2,708   & 5,278     & 1,433      & 1.949  & 0.81                                                     & papers       & citation    & research field        \\
CiteSeer                                                    & 6         & 3,327   & 4,552     & 3,703      & 1.368  & 0.736                                                    & papers       & citation    & research field        \\
PubMed                                                      & 3         & 19,717  & 44,324    & 500        & 2.248  & 0.802                                                    & papers       & citation    & research field        \\
\begin{tabular}[c]{@{}c@{}}Coauthor CS\end{tabular}       & 15        & 18,333  & 81,894    & 6,805      & 4.467  & 0.808                                                    & authors      & co-authors  & research field        \\
\begin{tabular}[c]{@{}c@{}}Coauthor Physics\end{tabular} & 5         & 34,493  & 247,962   & 8,415      & 7.189  & 0.931                                                    & authors      & co-authors  & research field        \\
\midrule
penn94                                                      & 2         & 41,554  & 1,362,229 & 5          & 32.782 & 0.47                                                     & peoples      & friends     & Gender  \\
\begin{tabular}[c]{@{}c@{}}arXiv-year\end{tabular}       & 5         & 169,343 & 1,166,243 & 128        & 6.887  & 0.222                                                    & papers       & citation    & year    \\
genius                                                      & 2         & 421,961 & 984,979   & 12         & 2.334  & 0.618                                                    & users        & followers   & Gender  \\
\begin{tabular}[c]{@{}c@{}}twitch-gamer\end{tabular}    & 2         & 168,114 & 6,797,557 & 7          & 40.434 & 0.545                                                    & Twitch users & followers   & Gender \\
\bottomrule
\end{tabular}
\end{table}

\section{Implementation Details}
\subsection{Model Implementation Details}
The implementation of \pjn{} basically follows the model we described in Section \ref{sec:ModelDesign}, in which we omit some details for simplicity. The complete implementation details of \pjn{} are listed in Table \ref{tab:impl}. Before extracting the ego-node feature, the raw input node features are fed into a Dropout layer. In the aggregated neighborhood features extraction, we reuse the intermediate results of different hops of neighbors to simplify the computation. For example, when computing $\hat{\mathbf{A}}^{i} \mathbf{H}_\text{ego}$, we use the dot product of $\hat{\mathbf{A}} \cdot (\hat{\mathbf{A}}^{i-1} \mathbf{H}_\text{ego})$ instead of computing the power of $\hat{\mathbf{A}}$. Moreover, because $\hat{\mathbf{A}}^{i-1} \mathbf{H}_\text{ego}$ has the same shape of $\mathbf{H}_\text{ego}$, computing $\hat{\mathbf{A}} \cdot (\hat{\mathbf{A}}^{i-1} \mathbf{H}_\text{ego})$ is always a sparse-dense matrix multiplication, which is more efficient than computing the power of $\hat{\mathbf{A}}$. Similarly, we use the same strategy to compute $\mathbf{H}_\text{strc}$. Because we use the original adjacency matrix instead of a normalized one in structure features extraction (for accuracy performance purposes), the magnitude of $\mathbf{A}^j$ will grow exponentially with $j$ increasing, which may cause the value out of range problem. Therefore, we utilize Batch Normalization layers to scale down the feature matrix after each adjacency matrix multiplication. In the feature fusion module, we dropout the fused features before activating it with ReLU.

\begin{table}[htb]
\centering
\caption{OGNN Model Implementation details}
\label{tab:impl}
\vskip 0.1in
\small
\tabcolsep=0.08cm
\begin{tabular}{l|l}
\toprule
Module & Implementation Details \\
\midrule
\midrule

Input Feature & $\mathbf{X}$\\
\midrule
Ego Node features & $\mathbf{H}_\text{ego} = \text{Linear}(\text{Dropout}(\mathbf{X}))$\\
Aggregated Neighborhood features & $\mathbf{H}_\text{agg} = \sum_{i=1}^{s_1} \hat{\mathbf{A}}^{i} \mathbf{H}_\text{ego}$\\
Graph Structure features  &$\mathbf{H}_\text{strc} = \sum_{j=1}^{s_2} \text{BN}^{j}(\mathbf{A}) \mathbf{W}_\text{strc}$,  $\text{BN}^{j}(\mathbf{A})=\underbrace{\text{BN}(\mathbf{A} \cdot \text{BN}(\dots \text{BN}  (\mathbf{A})))}_{s_2}$\\
\midrule
Feature Fusion & $\mathbf{H} = \text{ReLU}(\text{Dropout}(\pi_1 \mathbf{H}_\text{ego} + \pi_2 \mathbf{H}_\text{agg} + \pi_3 \mathbf{H}_\text{strc})), \quad \pi_i = \frac{\exp{p_i}}{\sum_{j=1}^{3} \exp{p_j}}$\\
\midrule
Prediction Head &$\mathbf{Y}_\text{pred} = \text{Softmax}(\mathbf{H} \mathbf{W}_\text{pred})$\\
\bottomrule
\end{tabular}
\end{table}

\subsection{Training Details}
To implement our bi-level optimization training scheme, we utilize two optimizers to train the model parameters $\mathbf{W}$ and the feature fusion parameters $\mathcal{P}$ respectively. The model parameters $\mathbf{W}$ are trained with an Adam optimizer \cite{kingma2014adam} $\mathcal{O}_1$ on the training dataset. The learning rate and weight decay rate of $\mathcal{O}_1$ is decided by the hyper-parameter settings, which is described in Section \ref{sec:hps}. On the other hand, the feature fusion parameters $\mathcal{P}$ are trained with another Adam optimizer $\mathcal{O}_2$ on the validation dataset with a fixed learning rate $0.01$. We train $\mathcal{P}$ for 10 epochs after training $\mathbf{W}$ for every 20 epochs. When training $\mathcal{P}$, we'll set the Dropout layers and Batch Normalization layers to evaluation mode.
We early stop the model if the validation accuracy does not increase for 100 epochs or the total number of training epochs reaches 3000.

\section{Hyper-parameter Settings} \label{sec:hps}
\label{app:hyperparameter}
In our experiments, we use the hidden channels ($d$), the propagation steps for aggregated neighborhood features extraction ($s_1$), the power of the adjacency matrix for graph structure features extraction ($s_2$), the learning rate $\eta$, the weight decay $\lambda$, and the feature normalization ($\nu$) as the hyper-parameters. For all the datasets, we perform grid search over the following hyper-parameter options:
\begin{align*} \label{eq:mplayer}
  d &\in \{64 \quad 128 \} \\
  s_1 &\in \{2 \quad 5 \quad 10 \quad 20 \} \\
  s_2 &\in \{1 \quad 2 \quad 5 \} \\
  \eta &\in \{0.01 \quad 0.001 \} \\
  \lambda &\in \{0.001 \quad 0.0005 \} \\
  \nu &\in \{\text{True} \quad \text{False} \}
\end{align*}

We also list the best hyper-parameter settings for all the real-world datasets in Table \ref{tab:hp}

\begin{table}[htb]
\centering
\caption{Best hyper-parameter settings for the real-world datasets.}
\label{tab:hp}
\vskip 0.1in
\small
\tabcolsep=0.08cm
\begin{tabular}{c|ccccccccc}
\toprule
             &\begin{tabular}[c]{@{}l@{}}arXiv\end{tabular}  & penn94 & 
             \begin{tabular}[c]{@{}l@{}}twitch\end{tabular} & genius & CiteSeer & PubMed & Cora  & \begin{tabular}[c]{@{}l@{}} CS\end{tabular} & \begin{tabular}[c]{@{}l@{}}  Physics\end{tabular} \\
Homophily    & 0.22  & 0.47   & 0.55    & 0.62  & 0.74    & 0.80  & 0.81  & 0.81    & 0.93  \\
\midrule
\midrule
$d$       & 64     & 64     & 128    & 128   & 128   & 128    & 128    & 64    & 128   \\
$s_1$     & 2      & 5      & 10     & 2     & 20    & 20     & 20     & 10    & 20    \\
$s_2$     & 2      & 2      & 2      & 5     & 2     & 5      & 5      & 5     & 2     \\
$\eta$    & 0.01   & 0.01   & 0.01   & 0.01  & 0.01  & 0.01   & 0.01   & 0.01  & 0.01  \\
$\lambda$ & 0.0005 & 0.0005 & 0.0005 & 0.001 & 0.001 & 0.0005 & 0.0005 & 0.001 & 0.0005\\
$\nu$     & False  & False  & False  & False & True  & True   & True   & False & False \\
\bottomrule
\end{tabular}
\end{table}

\section{Detailed Results} \label{sect:app-synthetic-details}

\subsection{Detailed Syntheic Dataset Results}

We list the detailed results of the experiments on the synthetic datasets in Table \ref{tab:syn-cora-results1} and \ref{tab:syn-cora-results2}, which includes standard deviation of the accuracy performance compared to Table \ref{tab:syn-cora-results}. Overall, \pjn{} outperforms the existing methods on most homophily settings with seven settings at top-1 and four at top-2, and pushes the best accuracy boundary for up to $2.76\%$.

\begin{table}[ht]
\centering
\caption{Test accuracy of different methods on the graphs with different homophily from 0 to 0.5 in \texttt{syn-cora} dataset. {\color{red} \textbf{Red}} and {\color{blue} blue} represent top-1 and top-2 ranking in terms of accuracy respectively.}
\label{tab:syn-cora-results1}
\vskip 0.1in
\small
\tabcolsep=0.1cm
\begin{tabular}{c|cccccc}
\toprule
synh   & 0                                                                                       & 0.1                                                                                     & 0.2                                                                                     & 0.3                                                                                     & 0.4                                                                                     & 0.5                                                                                                                                                                           \\
\midrule
\midrule
MLP	& $69.20_{\pm 1.92}$	  & $69.20_{\pm 1.92}$	 & $69.20_{\pm 1.92}$	 & $69.20_{\pm 1.92}$	 & $69.20_{\pm1.92}$	  & $69.20_{\pm 1.92}$		 \\
GCN	& $28.61_{\pm 2.38}$	 & $30.64_{\pm 1.58}$	 & $36.03_{\pm 1.15}$	 & $45.15_{\pm 2.18}$	 & $51.39_{\pm1.48}$	  & $65.04_{\pm 1.60}$	 	 \\
GAT	& $29.41_{\pm 2.20}$	 & $30.32_{\pm 2.52}$	 & $34.83_{\pm 1.67}$	 & $43.86_{\pm 1.44}$	 & $51.15_{\pm 1.25}$	 & $64.80_{\pm 1.81}$	  \\
DAGNN  & $34.32_{\pm 1.51}$	 & $39.49_{\pm 1.18}$	 & $45.01_{\pm 2.42}$	 & $54.48_{\pm 3.17}$	 & $60.51_{\pm 1.40}$	 & $72.36_{\pm 1.83}$	  \\
LINKX  & $72.09_{\pm 1.65}$	 & $70.54_{\pm 1.84}$	 & $69.76_{\pm 1.07}$	 & $70.13_{\pm 1.30}$	 & $71.34_{\pm 1.17}$	 & $74.53_{\pm 1.83}$	 	 \\
H2GCN  & {\color[HTML]{FF0000} $\mathbf{76.43_{\pm 1.27}}$} & {\color[HTML]{FF0000} $\mathbf{73.86_{\pm 3.05}}$} & {\color[HTML]{0000FF} $71.58_{\pm 1.75}$}	  & {\color[HTML]{0000FF} $72.17_{\pm 1.44}$}	  & {\color[HTML]{0000FF} $72.95_{\pm 0.76}$}	& {\color[HTML]{0000FF} $78.31_{\pm 1.96}$} 	 \\
MIXHOP & $39.44_{\pm 2.98}$	 & $38.95_{\pm 2.21}$	 & $41.05_{\pm 2.69}$	 & $48.93_{\pm 2.72}$	 & $55.09_{\pm 2.30}$	 & $64.75_{\pm 1.88}$	  \\
GPR-GNN & $67.86_{\pm 2.66}$	 & $61.96_{\pm 2.53}$	 & $61.21_{\pm 2.36}$	 & $64.69_{\pm 2.88}$	 & $67.67_{\pm 3.41}$	 & $74.56_{\pm 2.60}$	  \\
\pjn{}   & {\color[HTML]{0000FF} $72.49_{\pm 1.74}$}	  & {\color[HTML]{0000FF} $73.67_{\pm 1.61}$}	  & {\color[HTML]{FF0000} $\mathbf{73.11_{\pm 0.94}}$} & {\color[HTML]{FF0000} $\mathbf{74.32_{\pm 2.29}}$} & {\color[HTML]{FF0000} $\mathbf{75.71_{\pm 1.00}}$} & {\color[HTML]{FF0000} $\mathbf{79.49_{\pm 1.58}}$} 	 \\
\bottomrule
\end{tabular}
\end{table}

\begin{table}[ht]
\centering
\caption{Test accuracy of different methods on the graphs with different homophily from 0.6 to 1 in \texttt{syn-cora} dataset. {\color{red} \textbf{Red}} and {\color{blue} blue} represent top-1 and top-2 ranking in terms of accuracy respectively.}
\label{tab:syn-cora-results2}
\vskip 0.1in
\small
\tabcolsep=0.1cm
\begin{tabular}{c|ccccc}
\toprule
synh    & 0.6                                                                                     & 0.7                                                                                     & 0.8                                                                                     & 0.9                                                                                     & 1                                                                                       \\
\midrule
\midrule
MLP	& $69.20_{\pm 1.92}$	 & $69.20_{\pm 1.92}$	 & $69.20_{\pm 1.92}$	 & $69.20_{\pm 1.92}$	 & $69.20_{\pm 1.92}$	 \\
GCN	& $74.48_{\pm 1.36}$	 & $82.22_{\pm 2.20}$	 & $91.21_{\pm 1.15}$	 & $96.19_{\pm 0.65}$	 & $99.92_{\pm 0.18}$	 \\
GAT	& $74.34_{\pm 2.03}$	 & $81.45_{\pm 1.72}$	 & $90.46_{\pm 1.12}$	 & $95.79_{\pm 0.71}$	 & {\color[HTML]{FF0000} $\mathbf{100.0_{\pm 0.00}}$} \\
DAGNN  & $80.00_{\pm 1.26}$	 & $86.49_{\pm 2.13}$	 & $93.32_{\pm 1.45}$	 & {\color[HTML]{FF0000} $\mathbf{97.48_{\pm 0.31}}$} & {\color[HTML]{0000FF} $99.95_{\pm 0.07}$}	  \\
LINKX  & $77.16_{\pm 1.22}$	 & $80.35_{\pm 1.47}$	 & $83.30_{\pm 1.23}$	 & $87.59_{\pm 1.07}$	 & $89.60_{\pm 1.86}$	 \\
H2GCN  & {\color[HTML]{0000FF} $83.27_{\pm 1.25}$}	  & {\color[HTML]{FF0000} $\mathbf{87.43_{\pm 0.88}}$} & $92.09_{\pm 0.46}$	 & $97.00_{\pm 0.28}$	 & $98.98_{\pm 0.62}$	 \\
MIXHOP & $74.45_{\pm 1.44}$	 & $82.44_{\pm 2.69}$	 & $91.45_{\pm 1.11}$	 & $96.25_{\pm 0.89}$	 & {\color[HTML]{FF0000} $\mathbf{100.0_{\pm 0.00}}$} \\
GPR-GNN & $80.19_{\pm 1.57}$	 & $86.68_{\pm 0.69}$	 & {\color[HTML]{0000FF} $93.54_{\pm 1.36}$}	  & {\color[HTML]{0000FF} $97.45_{\pm 0.28}$}	  & {\color[HTML]{FF0000} $\mathbf{100.0_{\pm 0.00}}$} \\
\pjn{}   & {\color[HTML]{FF0000} $\mathbf{84.67_{\pm 1.30}}$} & {\color[HTML]{0000FF} $87.32_{\pm 1.25}$}	  & {\color[HTML]{FF0000} $\mathbf{93.70_{\pm 1.13}}$} & $94.75_{\pm 3.47}$	 & {\color[HTML]{0000FF} $99.95_{\pm 0.07}$}	 \\
\bottomrule
\end{tabular}
\end{table}

\subsection{Using Concatenation as Feature Fusion} \label{sec:app_ablation_concat}

Table \ref{tab:ablation_concat} shows the experimental results of using concatenation as feature fusion approach instead of summing up the features.

\begin{table}[ht]
\centering
\caption{Using concatenation as feature fusion.}
\label{tab:ablation_concat}
\vskip 0.1in
\scriptsize
\tabcolsep=0.08cm
\begin{tabular}{l|lllllllll|c}
\toprule
             &\begin{tabular}[c]{@{}l@{}}arXiv\end{tabular}  & penn94 & 
             \begin{tabular}[c]{@{}l@{}}twitch\end{tabular} & genius & CiteSeer & PubMed & Cora  & \begin{tabular}[c]{@{}l@{}} CS\end{tabular} & \begin{tabular}[c]{@{}l@{}}  Physics\end{tabular} & $\Delta$ Avg. \\
Homophily    & 0.22       & 0.47   & 0.55        & 0.62  & 0.74    & 0.80  & 0.81  & 0.81                                                   & 0.93                                                      & -          \\
\midrule
\midrule
\textbf{\pjn{}} & \textbf{56.53} & \textbf{85.44} & \textbf{66.07} & \textbf{89.72} & \textbf{69.82} & \textbf{79.98} & \textbf{83.50} & \textbf{93.15}   & \textbf{93.56}                                               & -          \\
w/ sum   & 53.45  \textsubscript{3.08\textdownarrow}    & 84.35 \textsubscript{1.09\textdownarrow} & 65.76  \textsubscript{0.32\textdownarrow}     & 87.64 \textsubscript{2.08\textdownarrow} & 67.72 \textsubscript{2.10\textdownarrow}  & 76.04 \textsubscript{3.94\textdownarrow} & 82.48 \textsubscript{1.02\textdownarrow} & 91.13  \textsubscript{2.02\textdownarrow}                                                 & 92.64   \textsubscript{0.92\textdownarrow}                                                     & 1.84\textdownarrow       \\
w/ concate & 56.55  \textsubscript{0.02\textuparrow}    & 82.14 \textsubscript{3.30\textdownarrow} & 65.87  \textsubscript{0.20\textdownarrow}      & 89.44 \textsubscript{0.28\textdownarrow} & 66.70 \textsubscript{3.12\textdownarrow}   & 74.70 \textsubscript{5.28\textdownarrow} & 81.10 \textsubscript{2.40\textdownarrow} & 89.73 \textsubscript{3.42\textdownarrow}                                                  & 92.74 \textsubscript{0.82\textdownarrow}                                                       & 2.09\textdownarrow      \\
\bottomrule
\end{tabular}
\end{table}

\end{document}